\documentclass[12pt, anon]{l4dc2024}

\usepackage{xcolor}  

\usepackage{algorithm}[1] 
\usepackage{mathtools}
\usepackage{epigraph}  
\usepackage{enumitem} 
\usepackage{algorithm}
\usepackage{algorithmic}
\usepackage{listings} 

\DeclareMathOperator*{\argmin}{arg\,min}
\DeclarePairedDelimiterX{\inp}[2]{\langle}{\rangle}{#1, #2}

    \def\ddefloop#1{\ifx\ddefloop#1\else\ddef{#1}\expandafter\ddefloop\fi}

    \def\ddef#1{\expandafter\def\csname c#1\endcsname{\ensuremath{\mathcal{#1}}}}
    \ddefloop ABCDEFGHIJKLMNOPQRSTUVWXYZ\ddefloop

    \def\ddef#1{\expandafter\def\csname s#1\endcsname{\ensuremath{\mathsf{#1}}}}
    \ddefloop ABCDEFGHIJKLMNOPQRSTUVWXYZ\ddefloop

    \def\ddef#1{\expandafter\def\csname b#1\endcsname{\ensuremath{\mathbb{#1}}}}
    \ddefloop ABCDEFGHIJKLMNOPQRSTUVWXYZ\ddefloop
    
\newtheorem{thm}{Theorem}
\newtheorem{lem}[thm]{Lemma}
\newtheorem{prop}[thm]{Proposition}
\newtheorem{claim}[thm]{Claim}

\newtheorem{coro}[thm]{Corollary}

\newtheorem{iden}{Identity}

\def\bbDelta{\boldsymbol{\Delta}} 

\title[Learning with Average Dependent Costs.]{ Online Decision Making with History-Average Dependent Costs (Extended)} 
\usepackage{times}

\author{%
\Name{Vijeth Hebbar} \Email{vhebbar2@illinois.edu} \AND \Name{C\'edric Langbort} \Email{langbort@illinois.edu}\\
 \addr Coordinated Science Lab, University
of Illinois at Urbana–Champaign, Urbana, IL 61801, USA.%
}

\begin{document}

\maketitle

\begin{abstract}
In many online sequential decision-making scenarios, a learner's choices affect not just their current costs but also the future ones. In this work, we look at one particular case of such a situation where the costs depend on the time average of past decisions over a history horizon. We first recast this problem with history dependent costs as a problem of decision making under stage-wise constraints. To tackle this, we then propose the novel Follow-The-Adaptively-Regularized-Leader (FTARL) algorithm. Our innovative algorithm incorporates \emph{adaptive regularizers} that depend explicitly on past decisions, allowing us to enforce stage-wise constraints while simultaneously enabling us to establish tight regret bounds. We also discuss the implications of the length of history horizon on design of no-regret algorithms for our problem and present impossibility results when it is the full learning horizon.  


\end{abstract}


\section{Introduction} \label{sec:intro}

In the classical online optimization framework, one seeks to study a multi-stage decision making process where a learner faces a series of cost functions $\{l^t\}_{t=1}^T$ over a time horizon $T$. The learner has no prior knowledge about the sequence of cost functions they will face but has to make decisions $x^t$ at stage $t$ using only the information about the cost functions faced in past stages. They then incur a loss $l^t(x^t)$ for their decision. The goal of the learner is to make decisions that result in low \emph{regret},  viz. the difference between the learner's cumulative cost and the cost of the best-in-hindsight decision. This setting has been widely studied and has been successfully applied in domains ranging from portfolio management [\cite{cover1991universal, blum1997universal}] and auctioning [\cite{bar2002incentive}] in economics to network routing [\cite{awerbuch2008online}] and control [\cite{abbasi2011adaptive}] in engineering. Readers are referred to \cite{hazan2022introduction} and \cite{cesa2006prediction} for an extensive survey. 


Deviating from this traditional setup, we seek to study the problem where the cost depends not only on the current action but also on the past actions. While such history dependence may in general show up in any number of ways, we restrict ourselves to the special case where costs (or payoffs) depend solely on the \emph{time average} of \emph{finitely} many past actions. This is motivated, for example, by the following scenarios. 

Consider a manufacturing facility that selects a product from its range to produce at each time step, with the aim of producing the most profitable item. For ease of exposition, let us suppose that the facility produces one unit of this chosen good per time step. Now we make the assumption that the manufacturer must `commit' to produce the chosen good for a fixed term of $H$ steps. Such `commitment' can encapsulate a myriad of scenarios where decisions have long-term effect, for example, in training workers for a specific production line, who are then employed on an $H$-period contract to produce that good. Regardless of the specific way in which this commitment manifests itself, it implies that the production initiated at a time step continues for the next $H$ time steps (in parallel with the production initiated in these later steps). 

Therefore, the proportion of each product in the total production at any time is simply the time-average of the production initiated in the past $H$ time steps. The manufacturer's reward -- quantified as profit per unit produced -- in a time step is tied to this average result of their past decisions. While the manufacturer may not know what profits they will enjoy for each product when initiating production, they can learn from historical data, and so the question arises

\textbf{Q1:} \emph{How would a learning agent make decisions when the profit they receive depends on the long-term average of their choices?}

Let us now look at another scenario that provides an alternate view on the problem of learning with history average dependent payoffs. Consider a media outlet generating content with a goal of achieving high viewership, while not knowing a priori what content is most desirable to consumers. Any realistic consumer with a bounded memory will base their viewership decision not just on the current content, but on the finite history of outputs by the media outlet. Indeed, a media house known primarily for coverage of financial news announcing they will feature an hour long interview with the new World Chess Champion will not attract as many viewers as a sports news channel doing the same. Thus, the media outlet's reward -- in this case their viewership -- depends on their \textit{reputation} i.e. the time-average of their past behaviour. 

Viewing \emph{reputation} as an averaged \emph{state} induced by the sequence of actions taken by the learner, the learner can be seen as trying to learn what \emph{reputation} to establish to enjoy high reward. In this view, the decision of the learner at a stage is constrained by their decision in past stages. Indeed the reputation of the media house in the eyes of a consumer with memory cannot be arbitrarily changed within short time frames. So, we ask the question

\textbf{Q2:} \emph{How would a learning agent make decisions when the decision across stages are coupled by constraints?}

These examples underline the dual aspects of historical influence: the enduring impact of past decisions versus the limitations imposed by past decisions on present choices. This duality in viewpoint is precisely captured in the two equivalent questions \textbf{Q1} and \textbf{Q2} posed above.
Finally, another natural question that arises from the consideration of the media outlet illustration above is how the length of the consumer's memory affects the media outlet's ability to learn about them. Simply put, if consumers remember everything, the media outlet might never overcome a `bad' reputation caused by its earlier decisions. So we then ask 

\textbf{Q3:} \emph{How should the length of the past horizon that affects current loss scale with $T$ to allow the learner to perform well?}


\subsection{Our Contributions and Paper Roadmap}

In this paper, we recast the problem of online sequential decision making with history-dependent payoffs as a problem of online optimization with stage-wise constraints. Specifically, in our work the dependence on history emerges solely through averaging and the cost in a stage relies explicitly on the average of past decisions. We also restrict ourselves to the class of online optimization problems with linear stage costs and decisions being picked from a finite-dimensional simplex. In other words, we consider the widely studied setup of \emph{Prediction-from-Expert-Advice} with the caveat that costs depend on the average of past decisions rather than the current decision alone. Section \ref{sec:setup} details the formalization of our framework. 

Drawing from the well-known Follow-The-Regularized-Leader (FTRL) [\cite{shalev2007primal}] class of algorithms, we propose the novel Follow-The-Adaptively-Regularized-Leader (FTARL) algorithm in Section \ref{sec:FTL}. While the traditional role played by the regularizer function in FTRL style algorithms is to avoid over-fitting while picking decisions [\cite{kalai2005efficient,shalev2012online}], we curate our regularizer to also ensure our constraints are satisfied. This is achieved by allowing the regularizer to depend on past cost functions as well as past decisions, earning it the moniker of an \emph{adaptive regularizer}. To the best of our knowledge this is the first such usage of history-dependent regularizers. Beyond novelty, these regularizers also provide a direct approach towards establishing regret bounds in this context, as we show in Theorem \ref{thm:FTAPL_regret_bd}.



With the aim of answering \textbf{Q3}, we first show  in Section \ref{sec:main_results} that when the stage costs depend on the full history of past decisions, no algorithm can achieve sublinear regret (Claim \ref{cla:full-history-example}). In the media outlet illustration from Section \ref{sec:intro}, this implies that if consumers have perfect memory, the media outlet cannot hope to achieve sublinear regret with respect to the viewership count. Using the tools developed in Section \ref{sec:FTL}, we then show that when this history dependence is restricted to a shorter horizon, specifically when $H\in o(T)$, we can achieve $\cO(\sqrt{TH})$ regret (Theorem \ref{thm:main_result}). In other words, when consumers are forgetful (i.e., have $o(T)$ memory), the media outlet can still learn what content to output with sublinear regret. A simple extension of our approach also shows that our approach is $H$-agnostic so long as we have an upperbound $H<\Theta \in o(T)$. In this case, we establish a $\cO(\sqrt{T\Theta})$ regret bound (Corollary \ref{cor:upperbound_on_H}). 



\subsection{Related Work} \label{sec:lit_rev}

Our problem is closely related to the framework of Online Convex Optimization with Memory (OCO-M) as analyzed by \cite{Anava2015OnlineMistakes}. In their work, they consider adversarially generated convex cost functions that can depend arbitrarily on a finite number of past decisions. Of the two approaches they present, the one better suited for the Expert Advice type of problem yields a regret bound of $\cO(\sqrt{HT\log(T)})$ in our case. By explicitly considering the nature of history dependence that shows up in our problem, our method offers an improvement by a logarithmic factor over theirs. A related line of work [\cite{Geulen2010RegretAlgorithm,Gyorgy2014Near-optimalCoding}] studies the Experts Advice problem when the cost of each expert also depends on their past actions. The challenge that arises then is that the cost incurred by the learner may be different from that of the picked expert as the two may have taken different actions in the past stages. In contrast, the experts in our setup have \textit{memoryless} costs, and the learner's cost alone depends on past actions.   




Online optimization with history dependent costs has garnered increased attention amongst the control community in recent years with an important application area being Online Linear Control (OLC) [\cite{cohen2018online,abbasi2011adaptive}]. In this framework, the cost at each stage depends on the state of a linear dynamical system and so implicitly depends on all past actions. The primary goal here is to arrive at a linear feedback controller (or some modified version of it [\cite{Agarwal2019OnlineDisturbances}]) at every stage to ensure low regret relative to the best-in-hindsight controller. On the other hand, in our problem, the meaningful notion of regret compares the performance of our algorithm with the performance of the best static action. 

Bearing question \textbf{Q2} in mind, there are multiple lines of research that try to incorporate constraints into an online optimization framework. Some works consider constraints that are adversarially generated  [\cite{kveton2008online,mannor2009online}] while some other consider a single long-term constraint connecting decisions across the learning horizon [\cite{wang2021online,altschuler2018online}]. 
In contrast, we consider a setup with stage-wise constraints that couple the decisions at each step with ones in the past. These constraints are known a-priori and we seek to design a decision making approach that explicitly accounts for them. In a similar spirit, \cite{badiei2015online} consider the problem of online optimization with ramp constraints i.e. known bounds on the magnitude of change in the decision across a step. To address this challenge, they consider a finite look ahead window on future costs. In contrast, we stick with the classical assumption in online optimization frameworks where only historical data is available when making decisions.

\section{Problem Setup} \label{sec:setup}
Let us now formalize our problem statement. Let $\{v^t\}_{t=1}^T \in \bbDelta_n$ denote the sequence of decisions made by the learner over a time horizon of length $T$. Here $\bbDelta_n$  denotes the $n$-dimensional simplex and $v_i^t$ corresponds to the weight given to action $i\in [n]$ at time $t$. In the case of Prediction from Expert Advice,  $v^t_i$ has the special interpretation of being the probability of choosing action $i$ at time $t$. We then define the \emph{time-averaged decision} $x^t$ as 

\begin{equation} \label{eq:aggregate_action}
    x^{t} = h^t(v^1,\dots,v^t) \triangleq  \begin{cases}
        \frac{1}{t} \sum_{\tau=1}^{t} v^\tau & t <  H \\
        \frac{1}{H} \sum_{\tau=t-H+1}^{t} v^\tau & t\geq  H. 
    \end{cases}
\end{equation}

Thus, $x_i^t$ corresponds to the average weight given to action $i$ over the past horizon of length (at-most) $H$. We can view the decision $v^t$ as an \emph{input} and the time averaged decision $x^t$ as the \emph{state} at time $t$. This view makes explicit the idea that $x^t$ is not independent of the past, but is instead generated through an update process. Formalizing this very idea we have
\begin{flalign}
     x^t & = y^{t-1} + \beta^t v^t \quad \forall t\geq 1\label{eq:constrained_action} 
\end{flalign}
where $\beta^t = \frac{1}{\min\{t,H\}}$ and $y^t$ is defined as
\begin{equation*} 
    y^{t} \triangleq \begin{cases}
        \frac{1}{t+1} \sum_{\tau=1}^{t} v^\tau & 1 < t < H \\
        \frac{1}{H} \sum_{\tau=t-H+2}^{t} v^t & t \geq H. 
    \end{cases}
\end{equation*}
with $y^0 \triangleq  \mathbf{0}$. This switch in the view from (\ref{eq:aggregate_action}) to (\ref{eq:constrained_action}) captures precisely the change in viewpoint from $\textbf{Q1}$ to $\textbf{Q2}$ as $x^t$ can now be viewed as being connected to the past (captured through $y^t$) by a constraint.

Let $\{g^t\}_{t=1}^T \subset \bR^n_{\leq 0}$ be the sequence of non-positive cost vectors faced by the learner. These cost functions may be generated adversarially but, if so, we assume that the adversary is oblivious, i.e., the cost functions do not adapt to the realized past decisions of the learner. The cost incurred by the learner at time $t$ in our model is then simply $\langle g^t,x^t \rangle$. Naturally, the goal of the learner is to incur a low total cost $\sum_{t=1}^T \langle g^t,x^t \rangle$. \\


There are two notions of regret we can consider. First, the standard notion of regret generally considered in the \emph{memoryless} case (i.e. in the event we could choose the decision $x^t$ at every step independent of the past) that is defined as 
\begin{flalign} \cR_T = \sum_{t=1}^T \big\langle g^t, x^t \big\rangle - \min_{x\in \bbDelta_n} \langle G^T, x \rangle \text{ where } G^t\triangleq \sum_{\tau=1}^t g^{\tau}. \label{eq:regret_OLO_simplex} \end{flalign}
At a first glance, this appears like a very strong notion of regret for our setup since our decisions are connected by the constraints in (\ref{eq:constrained_action}).  The second weaker notion of regret -- one that is routinely employed in  OCO-M literature  [\cite{Anava2015OnlineMistakes,Arora2012OnlineRegret}] -- is that of \emph{policy regret} and is defined for our setup as
\begin{flalign} 
\cR_{T,\text{Pol}} \triangleq \sum_{t=1}^T \langle g^t, x^t \rangle - \min_{v\in \bbDelta_n} \sum_{t=1}^T \langle g^t, h^t(v,\dots,v) \rangle. \label{eq:pol_regret}  
\end{flalign}
However, in the special case where dependence on past decisions is captured through an averaging process, we can state the following 
\begin{lem} \label{lem:polR=R}
    When decisions $\{x^t\}_{t=1}^T$ satisfy the form in (\ref{eq:aggregate_action}), $\cR_T = \cR_{T,\text{Pol}}$.  
\end{lem}
\begin{proof}
    This follows simply from noting that $h^t(v,\dots,v)=v$ from the definition in (\ref{eq:aggregate_action}).
\end{proof}
Guided by Lemma \ref{lem:polR=R}, in this paper, we will continue to work with the regret as defined in (\ref{eq:regret_OLO_simplex}). Since we are working with this stronger notion of regret, any regret bound we obtain allows us to compare the performance of our algorithm with the performance of the best static action under \emph{no} history dependence. We make an additional note that when an algorithm picks decisions $\{x^t\}_{t=1}^T$ in a stochastic manner, the appropriate metric to evaluate its performance is \emph{expected regret}. This is defined simply by taking expectations over the definition in (\ref{eq:regret_OLO_simplex}) and they are taken with respect to the randomness in the algorithm. For the purposes of brevity, in this paper we will henceforth refer to `expected regret' simply as `regret' and assume it is understood from context which notion we are employing.


Our goal in this paper is to design an algorithm that allows the learner to generate decisions $\{x^t\}_{t=1}^T$ (that take the form in (\ref{eq:aggregate_action})) while ensuring that regret defined in (\ref{eq:regret_OLO_simplex}) grows sub-linearly with $T$. To this end we will first develop some theory in the following section, which will motivate our algorithm and help us in analyzing its performance.

\section{Going Beyond Follow-the-Leader} \label{sec:FTL}
\subsection{Preliminaries on Follow-The-Leader type algorithms}
Let $\{l^t(\cdot)\}_{t=1}^T$ denote the sequence of the cost functions -- mapping decision set $\cX \subset \bR^n$ to $\bR$ -- faced by the learner over a horizon of length $T$. We define the sub-sequence of cost functions and learner's actions until time $t$ as the history $\cH^t=\{(l^\tau,x^\tau)\}_{\tau=1}^{t}$. In line with the standard online optimization framework, we assume that when making decision $x^{t+1}$ the learner only has access to $\cH^{t}$. 
Let us now look more closely at one class of algorithms that are routinely applied to solve online optimization problems: Follow-The-Leader (FTL) type algorithms. Consider first the canonical FTL algorithm that picks the decision $x^t$ at time t such that
\begin{equation} \label{eq:FTL_alg}
    x^t \in \argmin_{x\in\cX} \sum_{\tau=1}^{t-1} l^\tau (x) \; \forall t > 1, \quad x^1 \in \cX.
\end{equation}
Note that we have not described either the class of cost functions or the decision set $\cX$. Throughout this section, we only assume that the minimum in (\ref{eq:FTL_alg}) exists and in the event of multiple minimizers, one is picked arbitrarily. 
For this algorithm we have the following well-known result bounding the regret of the FTL algorithm [\cite{cesa2006prediction,kalai2005efficient}].       
\begin{thm} \label{thm:FTL_regret_bd}
When $\{x^t\}$ is generated according to (\ref{eq:FTL_alg}), 
\begin{equation} \label{eq:FTL_regret}
    \sum_{t=1}^T l^t(x^t) - \min_{x\in\cX} \sum_{t=1}^T l^t (x) \leq \sum_{t=1}^T l^t(x^t) - l^t(x^{t+1}).
\end{equation}
\end{thm}
While the FTL algorithm captured in (\ref{eq:FTL_alg}) leads to sublinear regret in some well structured online optimization problems, it can lead to linear regret even in some simple cases like when $l^t(\cdot)$ is linear [\cite{shalev2012online}]. Consequently, closely related methods like Follow-the-Regularized-Leader [\cite{shalev2007primal}] and Follow-The-Perturbed-Leader [\cite{kalai2005efficient,hannan1957approximation}] were developed to ensure sublinear regret in a wide variety of online learning setups. Taking a page out of this book we propose the \textit{Follow-The-Adaptively-Regularized-Leader} (FTARL) algorithm. 
\subsection{Follow-The-Adaptively-Regularized-Leader Algorithm}
Going beyond existing methods, our method picks a regularizer function that explicitly depends on the history $\cH^t$ faced by the learner, hence the term \emph{adaptive}. Mathematically, the learner makes decisions $\{x^t\}_{1}^T$ according to
\begin{subequations}
\begin{flalign}
    x^{t+1} & \in \argmin_{x\in\cX} \bigg(R^{t}(x,\cH^t) + \sum_{\tau=1}^{t} l^\tau (x)\bigg) \; \forall t \geq 1,  \label{eq:FTAPL_xt}
    \\ x^1 & \in \argmin_{x\in\cX} R^0(x). \label{eq:FTAPL_x0}
\end{flalign} \label{eq:FTAPL_alg}
\end{subequations} 
While we introduce additionally restrictions on the nature of the regularizers $R^t(\cdot,\cdot)$ in later sections, for now we only assume that the regularizer ensures the existence of minimizers in (\ref{eq:FTAPL_alg}). In (\ref{eq:FTAPL_alg}) the second argument of the regularizer $R^t(\cdot,\cdot)$ highlights the dependence on history, but for brevity's sake we drop the argument henceforth and assume this history dependence is implicit. We now present a regret bound theorem for the FTARL algorithm presented in (\ref{eq:FTAPL_alg}).

\begin{thm} \label{thm:FTAPL_regret_bd}
    \textbf{(FTARL Regret Bound)} Let $\{x^t\}_1^T$ be picked according to the algorithm presented in (\ref{eq:FTAPL_alg}), then for any $x'\in \cX$ 
    \begin{flalign}
        \sum_{t=1}^T l^t(x^t) - l^t (x') \leq \sum_{t=1}^T l^t(x^t) - l^t(x^{t+1}) + 
        \sum_{t=0}^{T-1} \big(R^t(x^{t+2}) - R^t(x^{t+1})\big) +  R^T(x')-R^T(x^{T+1}). \notag 
    \end{flalign}
    Indeed picking $x'$ as the best-in-hindsight strategy gives us a regret bound. \\
    \begin{proof}
        Let us define $\tilde{l}^0(\cdot)= R^0(\cdot)$ and 
        $\tilde{l}^t(\cdot)= l^t(\cdot) + R^t(\cdot)- R^{t-1}(\cdot)$ for $t\geq 1$. Then it is easy to see that running the FTARL algorithm in (\ref{eq:FTAPL_alg}) is equivalent to running the FTL algorithm from (\ref{eq:FTL_alg}) with cost functions $\tilde l$ at every time step starting from $t=0$. Then as a direct corollary of Theorem \ref{thm:FTL_regret_bd} we have \begin{equation}
            - \sum_{t=0}^T \tilde l^t(x') \leq - \min_{x\in \cX}  \sum_{t=0}^T \tilde l^t(x) \leq - \sum_{t=0}^T \tilde l^t(x^{t+1}) \label{eq:FTAPL_pf_ineq}
        \end{equation}
        The remainder of the proof is then an algebraic exercise involving expanding expressions for $\tilde l_t(\cdot)$ in both LHS and RHS of (\ref{eq:FTAPL_pf_ineq}), rearranging terms and finally, adding $\sum_{t=1}^T l^t(x^t)$ to both sides. 
    \end{proof} 
    \end{thm}
    By allowing the regularizer in (\ref{eq:FTAPL_alg}) to also depend on the past decisions, we can use it to explicitly enforce constraints that relate past decisions to current decisions. In the following section, we will illustrate this ability by designing the regularizer function so that the decisions $x^t$ taken by our FTARL algorithm at time $t$ takes form described in (\ref{eq:aggregate_action}) from Section \ref{sec:setup}. We will then employ Theorem \ref{thm:FTAPL_regret_bd} to analyze the performance of resulting algorithm.

\section{FTARL for Average Dependent Costs} \label{sec:main_results}



Let us begin by restating our objective as outlined at the end of Section \ref{sec:setup}. In doing so, we will stick with the viewpoint 
\textbf{Q2}. We want our learning agent to establish states $\{x^t\}_{t=1}^T$ with the aim of incurring regret, as defined in (\ref{eq:regret_OLO_simplex}), that grows sub-linearly with $T$. Additionally, any algorithm $\cA$ that the learner employs to achieve this goal must satisfy two properties.
\begin{enumerate}[noitemsep]
    \item First, it must ensure that $x^t$ respects the relation in (\ref{eq:constrained_action}) at every stage.
    \item Second, it must only rely on information $\cH^{t-1}$ that is available to the learner at time $t$. 
\end{enumerate}


We will design such an algorithm based on the FTARL algorithm framework introduced in Section \ref{sec:FTL}. Our first step then will be to design adaptive regularizers $\{R^t(\cdot)\}_{t=0}^T$.  

 


\subsection{Designing the Adaptive Regularizer} \label{sec:regu_design}

Let $Z\in\bR^n_+$ be a random vector with every element being picked i.i.d from an exponential distribution i.e. $Z_i \stackrel{i.i.d}{\sim} \exp{(\epsilon)}$. We then define the random sequence 

\begin{equation}
    v^{t}_* \in \argmin_{v\in \bbDelta_n} \langle G^{t-1} - Z , v \rangle \quad \forall t\geq 1 \label{eq:FTPL_seq}
\end{equation}
where $G^0 = \textbf{0}_n \in \bR^n $. Owing to the stochasticity in $Z$, $v^t_*$ is a random variable and is uniquely defined almost surely. In the event of non-uniqueness, we assume that one of the minimizers is picked arbitrarily. We then define our regularizer function as
\begin{subequations}
\begin{flalign}
    R^0(x)&=-\langle Z, x \rangle \label{eq:FTAPL_reg0} \\
    R^t(x)&= \delta\bigg(\frac{1}{2} \|x-y^t\|^2_2 -  \beta^{t+1} \langle v^{t+1}_* ,x\rangle\bigg) -  \langle G^t,x\rangle \label{eq:FTAPL_regt} \quad \forall t\geq 1
\end{flalign} \label{eq:FTAPL_reg}
\end{subequations} 
for all $t\geq 1$. First, note  that regularizers defined in (\ref{eq:FTAPL_reg}) guarantee the existence of minimizers in (\ref{eq:FTAPL_alg}) when $\cX=\bbDelta_n$ and $l^t(\cdot)=\langle g^t,\cdot \rangle$ and so, the FTARL algorithm is well-defined. 
Secondly, note that these regularizers depend only on the history $\cH^t$ and so the causality property of the corresponding FTARL algorithm is satisfied. The dependence on past decisions is captured through $y^t$ and the dependence on the cost functions is captured both explicitly, through $G^t$, and implicitly, through $v_*^t$. Finally, we make the following

\begin{remark} \label{rem:FTARL_gives_dyn}
    For all $t\geq1$, the decisions $\{x^t\}_{t=1}^T$ induced by the FTARL algorithm in (\ref{eq:FTAPL_alg}) (with regularizers as defined in (\ref{eq:FTAPL_reg})), satisfy the relation in (\ref{eq:constrained_action}).
\end{remark} 

Note that $v^t_*$, as defined in (\ref{eq:FTPL_seq}), can essentially be viewed as the sequence of outputs of a Follow-the-Perturbed-Leader (FTPL) algorithm [\cite{kalai2005efficient}] when the perturbation picked is $-Z$ and stage costs are \textit{memoryless}. Thus, our algorithm is closely connected to the FTL class of algorithms. In practice, our algorithm can be implemented without the optimization step in (\ref{eq:FTAPL_alg}) or any explicit consideration of the regularizers in (\ref{eq:FTAPL_reg}), as highlighted in Algorithm \ref{alg:FTARL}. But, as we will see in Section \ref{sec:reg_bd}, these regularizers play a critical role in our analytic approach for establishing regret bounds. 

\begin{algorithm}
\caption{FTARL for History-Average Dependent Costs}
\begin{algorithmic}[1]
\REQUIRE Learning rate $\epsilon > 0$
\STATE Draw perturbation $Z_i \stackrel{i.i.d}{\sim} \exp{(\epsilon)}$
\FOR{each round $t=1,2,\ldots,T$}
    \STATE Pick decision $v^t = \argmin_{v\in \bbDelta_n} \langle G^{t-1} - Z , v \rangle \quad$
    \STATE Update state $x^t = y^{t-1} + \beta^t v^t$ 
    \STATE Observe cost vector $g^t$; Incur the loss $\langle g^t, x_t \rangle$ 
\ENDFOR
\end{algorithmic}
\label{alg:FTARL}
\end{algorithm}

\subsection{The Challenge of Horizon Length} \label{sec:counter_ex}
Before arriving at regret bounds for the proposed algorithm, we first present a result that highlights the limitations faced by a learner in our problem setup. We will constructively show that no algorithm can guarantee regret that grows sublinearly with $H$. In other words, if the horizon $H$ over which decision averaging takes place scales linearly with $T$, 
we have no hope of designing an algorithm that achieves an $o(T)$ regret. Equipped with this insight, we then obtain our regret bound under a restriction on the length of this horizon. 

\begin{claim} \label{cla:full-history-example}
    Let $H \leq 0.8 T$ be an multiple of $4$ and define $T_s = T-\frac{H}{4}$. Consider the sequence of two-dimensional cost vectors $\{ g^t\}_{t=1}^T$ such that 
\begin{flalign*}
    g^t  = \begin{bmatrix}
        0 \\
        0
    \end{bmatrix} \; \forall t\leq T_s, \; \; 
    g^t  = \begin{cases}
        [-1 \quad 0]^T \;\;\forall t>T_s & w.p. \; 1/2\\
        [0 \quad -1]^T \;\; \forall t>T_s & w.p. \; 1/2
    \end{cases}  \; 
\end{flalign*}
Then for any (possibly randomized) algorithm $\cA$ that generates $\{x^t\}_{t=1}^T$ in accordance with the process in (\ref{eq:aggregate_action}) we have
\begin{flalign*}
    \bE[\cR_T] \geq  H/32
\end{flalign*}
where the expectation is taken over the randomness in the cost function (as well the randomness, if any, in the algorithm). 
\end{claim}
\begin{proof}
    Effectively, at $t=T_s$, a fair coin is flipped once and the cost vector for all subsequent time steps is determined based on the result of this flip. Let us denote by random variable $\theta$ the outcome of the coin flip. Indeed, if the learner could see this coin flip before hand they could easily achieve the best-in-hindsight cost of
    \begin{flalign}
        \min_{x\in \bbDelta_2} \langle G^T,x \rangle = \frac{-H}{4} \label{eq:pf_linear_regret_bih}   
    \end{flalign}
    by playing the best response to the costs incurred after the coin flip for all $T$ stages. 
    Recalling that $x_1^{T_s}+x_2^{T_s}=1$, let us first assume that the learning algorithm generates $x_1^{T_s}\leq 1/2$. With this assumption, from (\ref{eq:aggregate_action}), $\forall t>T_s$ we have 
    \begin{flalign}
        x^t_1 &  = x^{T_s}_1 +  \frac{1}{H} \bigg(\sum_{\tau = T_s+1}^t v^\tau_1 - \sum_{\tau = T_s-H+1}^{t-H} v^t_1 \bigg) \leq \frac{1}{2} + \frac{(t-T_s)}{H} \leq \frac{3}{4}.  \label{eq:pf_linear_regret}
    \end{flalign}
    Then taking expectations both over the coin flip and the randomization in the learning algorithm, the expected cost incurred by the learner after the coin flip is
    \begin{flalign}
        \bE_{\theta,\cA}\bigg[\sum_{t=T_s+1}^T \langle g^t, x^t \rangle\bigg|x_1^{T_s}\leq 0.5 \bigg] & = \frac{-1}{2} \sum_{T_s+1}^T \big(\bE_\cA[x_1^t|x_1^{T_s}\leq 0.5] + \bE_{\cA}[x_2^t|x_1^{T_s}\leq 0.5]\big) \notag\\
        & \stackrel{(a)}{\geq} \frac{7}{8} (T_s-T) = \frac{-7H}{32}. \notag    
    \end{flalign}
    where the inequality $(a)$ results from (\ref{eq:pf_linear_regret}) and from noting that $x_2^t\leq 1$ for all $t$. An identical lower bound can be obtained if we instead assumed $x_2^{T_s}< 1/2$ and consequently, by law of total expectation we have $$\bE_{\theta,\cA}\bigg[\sum_{t=1}^T \langle g^t, x^t \rangle\bigg] \geq \frac{-7H}{32}$$ 
    Comparing this with the cost of best-in-hindsight action from (\ref{eq:pf_linear_regret_bih}) gives us the required lower bound on regret.
\end{proof}

Note that in Claim \ref{cla:full-history-example}, the assumption $H\leq 0.8 T$ is made only for technical reasons and it is possible to generate similar examples where $\Omega(H)$ regret is guaranteed for any $H\leq T$. Nevertheless, this claim lends us the insight that when $H\in \Omega(T)$ the learner cannot hope to achieve regret sublinear in $T$. With the hope of arriving at sublinear regret algorithms for our problem, the natural regime to explore then is when $H$ scales sub-linearly with $T$. 
\subsection{Regret Analysis} \label{sec:reg_bd}
Inline with the discussion in the previous section, henceforth in this paper, we will assume $H\in o(T)$. 
We now present the main result of our work 

\begin{thm} \label{thm:main_result}
    Running algorithm (\ref{eq:FTAPL_alg}) with regularizers defined in (\ref{eq:FTAPL_reg}) (or, equivalently, Algorithm \ref{alg:FTARL}) with $\epsilon=\sqrt{\frac{4 (\log(n) + 1)}{M^2(T-H)(2+H)}}$  results in an expected regret 
    \begin{flalign*}
        & \bE[\cR_T] \leq 5MH +  4 M \sqrt{(T-H) (H+2) (\log(n)+1)},
    \end{flalign*} 
    where the expectation is taken over the distribution of $Z$ employed in (\ref{eq:FTPL_seq}) and $M$ is a bound on $\|g^t\|_\infty$.
\end{thm}


\begin{proof}
    Let $x^*$ be the best-in-hindsight decision. For a general $\epsilon$, invoking Theorem \ref{thm:FTAPL_regret_bd} gives us  
   \begin{flalign}
        \bE[\cR_T]  \leq \underbrace{\sum_{t=1}^T \bE[\inp{g^t}{x^t - x^{t+1}}]}_{(A)} +\underbrace{\sum_{t=0}^{T-1} \bE\big[R^t(x^{t+2}) - R^t(x^{t+1})\big]}_{(B)}\notag
        +  \underbrace{\bE\big[R^T(x^*)-R^T(x^{T+1})\big].}_{(C)}   \notag
    \end{flalign}
    Using bounds for terms $(A)$,$(B)$ and $(C)$ obtained in Appendix \ref{app:term_bd} we have 
    \begin{flalign}
     \bE[\cR_T]  \leq & 4MH + \delta T+ \epsilon M^2 (2+H) (T-H) +  \bE\big[\big\langle Z, x^{H+1}- x^* \rangle\big] + \bE\big[\big\langle Z, x^{1}- x^2 \rangle\big]. \label{eq:regret_intermidiate} 
    \end{flalign} 
    The last two terms in the RHS above can further be upperbounded as
    \begin{flalign*}
    \bE\big[\big\langle Z, x^{H+1}- x^* \rangle\big] + \bE\big[\big\langle Z, x^{1}- x^2 \rangle\big] & \leq \bE\big[\big\|Z\|_\infty, \|x^{H+1}\|_1+ \|x^*\|_1+\|x^{1}\|_1+ \|x^2\|_1 \rangle\big] \\
    & \stackrel{(a)}{=}  \frac{4 H_n}{\epsilon} \leq \frac{4}{\epsilon} (\log(n)+1) 
    \end{flalign*}
    where equality $(a)$ results because the n-th order statistic of $n$ i.i.d exponential RVs with parameter $\epsilon$ has expected value $\frac{H_n}{\epsilon}$ where $H_n$ is the n-th harmonic number.
    Substituting this result back in (\ref{eq:regret_intermidiate}) and  picking the provided $\epsilon$ and $\delta=MH/T$ gives us our result. 
\end{proof}

Some comments on our result are in order. First, we note that our regret bound is linear in $H$. This was expected and agrees with the $\Omega(H)$ regret we obtained for the example in Section \ref{sec:counter_ex}. Accordingly, we began this section with the underlying assumption that $H\in o(T)$ which indicates that our regret bound is $\cO(M\sqrt{TH \log(n)})$. Indeed when $H=1$ our problem reduces to the Expert Advice problem and we recover the well established and tight $\cO(M\sqrt{T\log(n)})$ regret bound [\cite{cesa2006prediction}]. 

Recall that the notion of regret we consider compares the performance of our algorithm with that of the best-in-hindsight decision under no history dependence as highlighted in Section \ref{sec:setup}. The bounds in Theorem \ref{thm:main_result} then indicate that in reconciling with this history dependence the penalty we pay appears as a multiplicative factor that scales as $\sqrt{H}$. 



Finally, note that in Theorem \ref{thm:main_result}, the value of $\epsilon$ we picked depended explicitly on the $H$ which translates to our algorithm (\ref{eq:FTAPL_alg}) also depending on $H$. But knowledge of $H$ may be an unrealistic assumption in many cases - in the illustrative scenario from Section \ref{sec:intro} for instance, the media house may not have an estimate of the length of memory amongst their consumers. Fortunately, it suffices to know a suitable upper bound $\Theta$ on $H$ to make the following 

\begin{coro} \label{cor:upperbound_on_H}
    Running algorithm (\ref{eq:FTAPL_alg}) with regularizers defined in (\ref{eq:FTAPL_reg}) (or, equivalently, Algorithm \ref{alg:FTARL}) with $\epsilon=\sqrt{\frac{4 (\log(n) + 1)}{T M^2\Theta}}$  allows us to bound the expected regret as 
    \begin{flalign*}
        & \bE[\cR_T] \leq 5M \Theta +  4 M \sqrt{T\Theta(\log(n)+1)}.
    \end{flalign*} 
    where the expectation is taken over the distribution of $Z$ employed in (\ref{eq:FTPL_seq}) and $M$ is a bound on $\|g^t\|_\infty$.
\end{coro}

\section{CONCLUSIONS \& FUTURE WORK} 

In this paper, we considered a specific case of online decision making where the stage costs depend on the average of past decisions. By converting it into a online learning problem with stage-wise constraints, we were able to apply our novel FTARL algorithm to obtain tight regret bounds for this problem. The success of our approach in being able to handle constraints in online learning problems brings up a natural direction of future work. It remains to be seen how our algorithm can be adapted to handle constraints that crop up in other online learning scenarios. 

We also restricted ourselves to learning decisions that lie in a simplex with signed cost functions. Broadening our scope to include general convex decision set and non-linear convex cost function is an important avenue for future work. In our work, we also weigh every decision in the finite history horizon equally. However, in many realistic scenarios, past decisions may have less influence on the current cost compared to the current decision. Exploring weighted averaging within our framework is another interesting line of study to explore.     





\appendix
\section{Some useful notation, lemmas and identities}
Recalling that in our algorithm $v_*^t$ are generated according to (\ref{eq:FTPL_seq}), which is a linear program, every $v_*^t$ must lie at the vertex of the constraint set $\bbDelta_n$. Defining, $i^t\triangleq argmin_i \{G^t_i-Z_i\}$ we can then claim $v^t_* = e_{i_{t-1}}$ where $e_j$ is the $j^{th}$ basis vector in $\bR^n$. We now list some identities and results.

\begin{iden} \label{id:diff_state}
    For decisions $\{x^t\}_{t=1}^T$ made according to (\ref{eq:FTAPL_alg}) with regularizers defined in (\ref{eq:FTAPL_reg}) we have
    \begin{flalign}
        && x^{t+1} - x^{t} & \stackrel{a.s}{=} \frac{1}{H} (e_{i_{t}} - e_{i_{t-H}}) \tag{I.1} \quad \forall t\geq H \label{id:del_x_after_h} & \\
        && x^{t+1} - x^{t} & \stackrel{a.s}{=} \frac{1}{t+1} (e_{i_{t}} - x^t) \tag{I.2} \quad \forall t<H \label{id:del_x_before_h} & \\
        \text{which in turn gives us} && \quad  \|x^{t+1} - x^{t}\|_1 &\leq \frac{2}{\min\{t,H\}}. &\tag{I.3} \label{id:diff_x_bd}
    \end{flalign}
\end{iden}

\begin{claim} \textbf{(Probability of change in leader across $H$ steps.)} \label{cla:prob_leader_change}
For $t\geq H$,
$$\bP(x^t \neq x^{t+1}) \leq \epsilon m H.$$  
\end{claim}
\begin{proof}\textbf{:}
By Identity \ref{id:del_x_after_h} and law of total probability we have
\begin{flalign}
    \bP(x^t = x^{t+1}) = \sum_{j=1}^n \bP(i^t = j| i^{t-h} = j) \bP(i^{t-h} = j) \geq \exp{(-\epsilon MH)}. 
\end{flalign}
where the lower bound $\bP(i^t = j| i^{t-h} = j)\geq \exp{(-\epsilon MH)}$ is obtained by a procedure very similar to the one taken for the Expert Advice Problem (Theorem 5.10) in \cite{hazan2022introduction}. This gives us 
$$\bP(x^t \neq  x^{t+1}) \leq 1- \exp{(-\epsilon MH)} \leq \epsilon MH.$$
\end{proof}

\begin{prop}\textbf{(Difference in value of leader.)} \label{prop:diff_leader_value} 
Consider any time window $\cW = [t-\Theta,t]$ of length $\Theta$ and let $\cL^t$ denote the set of \textit{leaders} for this time window, i.e. 
$$\cL^t=\{l\in [n]|l = i_\tau \text{ for some } \tau \in \cW\}.$$
When $|\cL^t|\geq 2$, for all $j,k \in \cL^t$ we have
\begin{flalign*}
    |G_j^\tau -Z_j - (G_k^\tau -Z_k)| \leq M\Theta  \quad \forall \tau\in\cW. 
\end{flalign*}
\end{prop}
\begin{proof}
By definition of $\cL^t$ we can find time steps $\tau_j,\tau_k\in \cW$ such that 
\begin{flalign}
    G_j^{\tau_j} - Z_j - G_k^{\tau_j} + Z_k \leq 0 \label{eq:pf_j_leader}\\
    G_k^{\tau_k} - Z_k - G_j^{\tau_k} + Z_j \leq 0 \label{eq:pf_k_leader}.
\end{flalign}
For any $\tau \in \cW$, using (\ref{eq:pf_j_leader}) we have,
\begin{flalign*}
    G_j^{\tau} - Z_j - G_k^{\tau} + Z_k  \leq \underbrace{G_j^{\tau} - G_j^{\tau_j}}_{(a)} + \underbrace{G_k^{\tau_j} - G_k^{\tau}}_{(b)} \leq M|\tau-\tau_j| \leq M\Theta
\end{flalign*}
Note that in the expression above only one of $(a)$ and $(b)$ can be positive as $G^t_i$ is decreasing in $t$ for every $i$. Both of them are independently upper-bounded by $M|\tau-\tau_j|$ since $\|g^t\|\leq M$ for all $t$.  A similar argument can be used to show 
$G_k^{\tau} - Z_k - G_j^{\tau} + Z_j \leq M \Theta$ for every $\tau \in \cW$ using (\ref{eq:pf_k_leader}). This gives us the required result. 
\end{proof}

Finally, we provide a justification for Remark \ref{rem:FTARL_gives_dyn}. 

\begin{proof}\textbf{for Remark \ref{rem:FTARL_gives_dyn}} \\
    By definition of $v^1_*$ in (\ref{eq:FTPL_seq}), we have $v^1_* \in \argmin_{x \in \bbDelta_n} R^0(x)$. By picking $x^1 = v^1_*$ and noting that $y^0=0$ and $\beta^1=1$ we see that $x^1$ satisfies the relation in (\ref{eq:constrained_action}) for $t=1$. For $t>1$ we have 
    $$x^{t+1} \in \argmin_{x\in \bbDelta_n} R^t(x) + \langle G^t, x \rangle {=}  \argmin_{x\in \bbDelta_n} \delta\bigg(\frac{1}{2} \|x-y^t\|^2_2 -  \beta^{t+1} \langle v^{t+1}_* ,x\rangle\bigg) = \{ y^t + \beta^{t+1} v_*^{t+1}\}. $$
    The last equality was obtained by taking gradient of the objective with respect to $x$ and setting it equal to zero. Note that $y^t + \beta^{t+1} v_*^{t+1}$ lies in $\bbDelta_n$ by definition of $y^t$, $\beta^{t+1}$ and $v_*^{t+1}$ so it is indeed a minimizer. Thus, the decision $x^{t+1}$ picked by our algorithm satisfies the condition in (\ref{eq:constrained_action}) for all $t\geq1$.  
\end{proof}


\section{Bounding terms (A), (B) and (C)} \label{app:term_bd}
\subsection{Bounding term (A)}
\begin{flalign*}
    & \sum_{t=1}^T \bE[\inp{g^t}{x^t - x^{t+1}}] \rangle \\ & \stackrel{(a)}{\leq}  \sum_{t=1}^{H} \bE[\|g^t\|_\infty \|x^t - x^{t+1}\|_1] + \sum_{t=H+1}^{T} \bE\big[\inp{g^t}{x^t - x^{t+1}}\big|x^t\neq x^{t+1}\big] \bP(x^t\neq x^{t+1}) \\
    & \stackrel{(b)}{\leq} 2 M H  
 + \sum_{t=H+1}^{T} \bE\big[\|g^t\|_\infty\|x^t - x^{t+1}\|_1\big|x^t\neq x^{t+1}\big] \bP(x^t\neq x^{t+1}) \\
    & \stackrel{(c)}{\leq}  2MH + \sum_{t=H+1}^{T} \bigg(M\frac{2}{H}\bigg) \epsilon M H = 2MH + 2 (T-H) \epsilon M^2
\end{flalign*}
where $(a)$ results from H\"older's inequality and the law of total expectation. Inequality $(b)$ results from the boundedness of $g^t$, the fact that $x^t$ lies in the n-simplex and H\"older's inequality. Inequality $(c)$ results from the boundedness of $g^t$, identity \ref{id:diff_x_bd} and Claim \ref{cla:prob_leader_change}. 
\subsection{Bounding term (B)} \label{app:bound_C}
Dropping the first term we have
\begin{flalign}
    & \sum_{t=1}^{T-1} \bE\big[R^t(x^{t+2}) - R^t(x^{t+1})\big] \notag\\
     = & \sum_{t=1}^{T-1} \bE\bigg[\frac{\delta}{2} \big(\|x^{t+2}- y^t\|^2-\|x^{t+1}- y^t\|^2\big) - \delta\beta^{t+1} \langle v^{t+1}_*,x^{t+2}-x^{t+1} \rangle - \langle G^t,x^{t+2}-x^{t+1}\rangle \bigg] \notag \\
    \stackrel{(a)}{=} & \sum_{t=1}^{T-1} \bE\bigg[\frac{\delta}{2} \langle x^{t+2} - x^{t+1}, x^{t+2} + x^{t+1} - 2 y^t \rangle \notag  -  \delta \beta^{t+1} \langle v^{t+1}_*,x^{t+2}-x^{t+1} \rangle - \langle G^t,x^{t+2}-x^{t+1}\rangle \bigg] \notag \\ 
    \stackrel{(b)}{=} & \sum_{t=1}^{T-1} \bE\bigg[\frac{\delta}{2} \|x^{t+2} - x^{t+1}\|^2_2  - \langle G^t,x^{t+2}-x^{t+1}\rangle \bigg] \stackrel{(c)}{\leq} \delta (T-1) - \sum_{t=1}^{T-1} \bE\bigg[\langle G^t,x^{t+2}-x^{t+1}\rangle \bigg] \label{eq:bounding_C_b4_split}
\end{flalign}
where $(a)$ results from using the identity $\|\mu\|_2^2 - \| \nu\|^2_2 = \langle \mu-\nu,\mu+\nu\rangle $. Equality $(b)$ results simply from (\ref{eq:constrained_action}) and inequality $(c)$ results from noting that the diameter of a simplex is $\sqrt{2}$.  
Then we can bound the second term in the RHS of (\ref{eq:bounding_C_b4_split}) until $t=H$ as, 
\begin{flalign*}
    \sum_{t=1}^{H-1} \bE\bigg[\langle G^t,x^{t+1}-x^{t+2}\rangle \bigg]  
    \stackrel{(d)}{\leq} \sum_{t=1}^{H-1} \bE\bigg[\|G^t\|_\infty \|x^{t+2}-x^{t+1}\|_1 \bigg] \stackrel{(e)}{\leq}   \sum_{t=1}^{H-1} \frac{Mt}{t+1} \stackrel{}{\leq} MH 
\end{flalign*}
where $(d)$ results from employing H\"older's inequality, inequality $(e)$ results from applying Identity \ref{id:diff_x_bd} and from the boundedness of $g^t$.  We now obtain an upper-bound on the remaining terms in the RHS of (\ref{eq:bounding_C_b4_split}). Performing a change of index and adding and subtracting bold faced terms we have 
\begin{flalign*}
     & \sum_{t=H+1}^{T}\bE\bigg[ - \langle G^{t-1}\mathbf{-Z},x^{t+1}-x^{t}\rangle \bigg] + \mathbf{\bE\bigg[\bigg\langle Z, \sum_{t=H+1}^{T} x^{t}-x^{t+1}\bigg\rangle\bigg]} \\
     \stackrel{(f)}{=} & \sum_{t=H+1}^{T} \bE\bigg[- \langle G^{t-1}-Z,x^{t+1}-x^{t}\rangle \bigg| x^{t+1}\neq x^{t} \bigg] \bP(x^{t+1}\neq x^{t}) + \bE[\langle Z,  x^{H+1}-x^{T+1}\rangle]\\
     \stackrel{(g)}{\leq} &\sum_{t=H+1}^{T} \frac{1}{H}\bE\big[ G^{t-1}_{i^{t-H}}-Z_{i^{t-H}} - G^{t-1}_{i^{t}}+Z_{i^{t}}\big|i_{t}\neq i_{t-H} \big]\bP(x^{t+1}\neq x^{t}) + \bE[\langle Z,  x^{H+1}-x^{T+1}\rangle]\\
     \stackrel{(h)}{\leq} &\sum_{t=H}^{T-1} M \bP(x^{t+1}\neq x^{t+2}) + \bE[\langle Z,  x^{H+1}-x^{T+1}\rangle] \stackrel{(i)}{\leq} \epsilon M^2 H(T-H)  + \bE[\langle Z,  x^{H+1}-x^{T+1}\rangle]
\end{flalign*}
where $(f)$ results from applying the law of total expectation and simplifying a telescoping series. Inequality $(g)$ results from identity \ref{id:del_x_after_h}. Inequality $(h)$ results from noting that $i^{t},i^{t-H} \in \cL^{t}$ in applying Proposition \ref{prop:diff_leader_value} over the window $[t-H,t]$ and $(i)$ results from Claim \ref{cla:prob_leader_change}.  Accounting for the first term we dropped earlier gives us the following upper bound on term (C), 
\begin{flalign}
    \text{(C)} \leq &  \delta T + MH + \epsilon M^2 H (T-H) + \bE[\langle Z,  x^{H+1}-x^{T+1}\rangle]+\bE[\langle Z,  x^{1}-x^{2}\rangle]. \notag
\end{flalign}

\subsection{Bounding term (C)}
\begin{flalign*}
    \bE\big[R^T(x^*)-R^T(x^{T+1})\big] \stackrel{(a)}{=} & `\bE\bigg[\frac{\delta}{2} \|x^* - x^{T+1}\|^2_2  - \langle G^T,x^*-x^{T+1}\rangle \bigg] \\
    \stackrel{(b)}{\leq} & \delta - \bE\big[\big\langle G^T-Z, x^* - x^{T+1} \rangle\big] - \bE\big[\big\langle Z, x^* - x^{T+1} \rangle\big]   
\end{flalign*}
where $(a)$ was obtained similarly to the initial steps in Section \ref{app:bound_C}. Inequality $(b)$ was obtained by noting that the diameter of a probability simplex is $\sqrt{2}$ and by adding and subtracting terms. Let us define $i^* = \argmin_i \{G^T_i\}$ allowing us to redefine $x^*=e_{i^*}$. Then, employing $(\ref{eq:aggregate_action})$ and adding and subtracting $e_{i^T}$ we have 
\begin{flalign*}
    -\bE\big[\big\langle G^T-Z, x^* - x^{T+1} \rangle\big] = & \frac{1}{H} \sum_{t=T-H+1}^T \bE\big[\big\langle G^T-Z, e_{i_T} - e_{i^*} + e_{i_t}  - e_{i_T} \rangle\big] \\
    \stackrel{(c)}{\leq} & \frac{1}{H}\sum_{t=T-H+1}^T \bE\big[G^T_{i_t}-Z_{i_t}-G^T_{i_T}+Z_{i_T} \big] \stackrel{(d)}{\leq}  MH 
\end{flalign*}
where $(c)$ results from noting that $\langle G^T-Z, e_{i_T} \rangle = \min_{x\in \bbDelta} \langle G^T-Z, x \rangle \leq \langle G^T-Z, e_{i^*} \rangle$. Inequality $(d)$ results from noting that $i^{t},i^{T} \in \cL^{T}$ in applying Proposition \ref{prop:diff_leader_value} over the window $[T-H,T]$. Putting everything together we have the bound,
$$\bE\big[R^T(x^*)-R^T(x^{T+1})\big] \leq \delta +M H +  \bE\big[\big\langle Z, x^{T+1}- x^* \rangle\big]. $$

\section{Experimental Results}

We now present some simulation results where we apply our \textit{Follow-The-Adaptively-Regularized-Leader} (FTARL) algorithm to learn from synthetically generated data. For performance comparison, we also implement the \textit{low-switch algorithm} (LSA) for OCO-M by \cite{Anava2015OnlineMistakes}. 
\subsection{Methods}
For the simulation, we synthetically generate three classes of cost vectors $\{g^t\}_{t=1}^T$. 
\begin{enumerate}
    \item \textit{Identical Stochastic Cost Functions} \textbf{(StocId)} - At every time step the cost of each action is picked independently and uniformly from $[0,1]$ i.e. $g_i^t \stackrel{i.i.d}{\sim} U[0,1]$ for all $t$. 
    \item \textit{Heterogenous Stochastic Cost Functions} \textbf{(StocHet)} - Once at the beginning, an interval  $[a_i,b_i] \subseteq [0,1]$ is randomly generated for every action $i$ and the cost for that action is picked uniformly at random from this interval for every time step. In other words, $g_i^t \stackrel{i.i.d}{\sim} U[a_i,b_i]$ for all $t$
    \item \textit{Cyclic Cost Functions} \textbf{(Cyc)} - From a finite list of cost vectors one is picked in a cyclical pattern. Each picked cost vector is held constant for a fixed period $L$ before moving to the next. In our simulations, this list of cost functions is simply the set 
    $$\{-e_i | i\in [n], e_i \text{ is a standard basis vector in $\bR^n.$}  \}.$$    
\end{enumerate}

\begin{figure}
    \centering
    \subfigure{\includegraphics[width=0.49\textwidth]{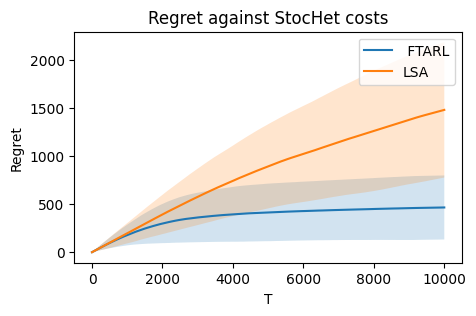} \label{fig:stocHet_LSA}} 
    \subfigure{\includegraphics[width=0.49\textwidth]{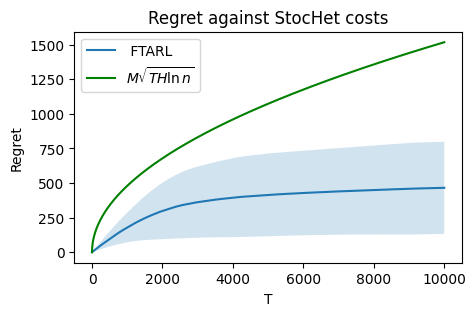}
    \label{fig:stocHet_BD}}
    \subfigure{\includegraphics[width=0.49\textwidth]{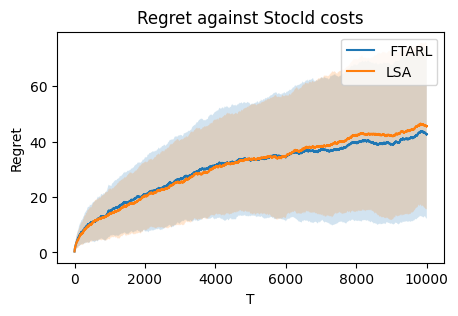} \label{fig:stocID_LSA}} 
    \subfigure{\includegraphics[width=0.49\textwidth]{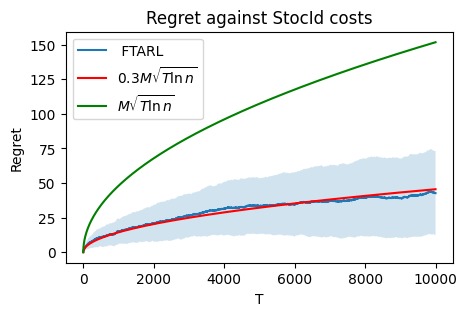} \label{fig:stocID_BD}}
    \subfigure{\includegraphics[width=0.49\textwidth]{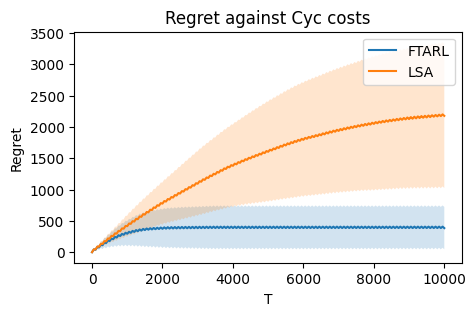} \label{fig:cyc_LSA}} 
    \subfigure{\includegraphics[width=0.49\textwidth]{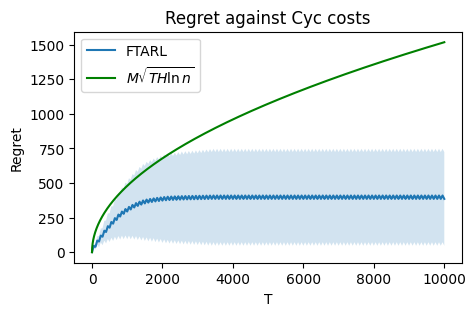} \label{fig:cyc_BD}} 
    \caption{\textbf{(Experimental results on synthetic data.)} 
     Figures 1(a) and 1(b) showcase the performance of our algorithm against \textbf{StocHet} cost functions. Figure 1(a) compares the average regret incurred by our algorithm against that incurred by the LSA algorithm by \cite{Anava2015OnlineMistakes}, while Figure 1(b) compares it with a term similar to the regret bound obtained in Theorem \ref{thm:main_result}. Figures 1(c),(d) and 1(e),(f) show similar comparisons for \textbf{StocId} and \textbf{Cyc} cost functions respectively.}
    \label{fig:regret_comparison}
\end{figure}

We run all our simulations (source code at github.com/vijeth27/LearningWithHist.git) for $T=10000$ learning horizon and the length of history horizon over which decisions are averaged is fixed at $H=100$. Note that this history horizon is of the order $\sqrt{T}$. In the event we are generating the cost functions in a cyclic manner (Cyc) we set $L=50$. Finally, we run $S=100$ iterations with each class of cost function and the average regret incurred by both FTARL and LSA are plotted in Figure \ref{fig:regret_comparison}. 

\subsection{Results}

We see in Figures \ref{fig:regret_comparison}(a,e) that our FTARL algorithm out-performs the LSA in the case of both \textbf{StocHet} and \textbf{Cyc} cost functions. This weaker performance of LSA parallels the weaker $\cO(\sqrt{T \log T})$ regret guarantees (compared to our $\cO(\sqrt{T})$ bound) obtained for it by \cite{Anava2015OnlineMistakes}. But the stronger performance of FTARL algorithm should not come as a surprise as our approach explicitly accounts for the specific averaging nature of history dependence, while LSA makes no such considerations. 

On the other hand, in Figure \ref{fig:regret_comparison}(c) we see that the performance of both FTARL and LSA are identical in the case of \textbf{StocID} cost functions. Note that in this scenario, every element of the cost vector $g^t$ has identical value in expectation and so, for any element $x^t \in \bbDelta_n$ the expected value of the cost $\langle g^t,x^t \rangle$ remains constant. This in turn indicates that every algorithm that generates $x^t$'s from the set $\bbDelta_n$ will achieve identical performance. It also means that our FTARL algorithm that respects the process in (\ref{eq:constrained_action}) achieves regret indistinguishable from an algorithm that picks $x^t$ in an unconstrained fashion. In Figure \ref{fig:regret_comparison}(d), we illustrate a comparison between our proposed algorithm and the regret bound obtained in the no-history dependence scenario for \textbf{StocID} cost functions. While this plot hints that the growth in regret of our algorithm can be $\Omega(\sqrt{T})$, it also indicates that effect of history dependence may not always act to scale this regret by a factor of $\sqrt{H}$. More work is needed to ascertain the tightness of our regret bound in $H$.

\acks{This work was supported by the ARO MURI grant W911NF-20-0252
(76582 NSMUR).}

\bibliography{l4dc2024}

\end{document}